\documentclass{article}
\usepackage{mathptmx}
\usepackage{comment}
\long\def\comment#1{}

\usepackage[acronym]{glossaries}
\usepackage{wasysym}
\usepackage{enumerate}
\usepackage[english]{babel}
\usepackage[linesnumbered,ruled,slide,boxed,vlined]{algorithm2e}
\usepackage{graphicx}
\usepackage{float}
\usepackage{xcolor}
\usepackage{xspace}
\usepackage{accents}
\usepackage{amsfonts}
\usepackage[hidelinks]{hyperref}
\usepackage[citestyle=numeric,bibstyle=numeric]{biblatex}

\usepackage{authblk}
\usepackage{amsthm}
\usepackage{sectsty}
\theoremstyle{definition}
\sectionfont{\fontsize{10}{10}\selectfont}
\author[1]{Spencer Killen }
\author[1]{Jia-Huai You}
\affil[1]{University of Alberta}

\AtEveryBibitem{\clearfield{note}}
\AtEveryBibitem{\clearfield{url}}
\AtEveryBibitem{\clearfield{isbn}}
\AtEveryBibitem{\clearfield{issn}}
\AtEveryBibitem{\clearfield{doi}}
\AtEveryBibitem{\clearfield{urlyear}} 
\AtEveryBibitem{\clearfield{timestamp}}

\newcommand{\head}{\textit{head}}
\newcommand{\body}{\textit{body}}
\newcommand{\bodyp}{\textit{$body^{+}$}}
\newcommand{\bodyn}{\textit{$body^{-}$}}
\newcommand{\union}{\cup}
\newcommand{\intersect}{\cap}
\newcommand{\setcomp}{\textrm{ }|\textrm{ }}

\newcommand{\lxor}{\oplus}

\newcommand{\mknf}{\text{{MKNF} }}
\newcommand{\MKNF}{\textmd{\rm {MKNF}} }

\newcommand{\Not}{{\textit {\bf  not}\,}}

\newcommand{\bfK}{{\textit {\bf{K}}\,}}

\newcommand{\Katom}{\textbf{K}-atom}
\newcommand{\Katoms}{\textbf{K}-atoms}

\newcommand{\K}{{\cal K}}
\renewcommand{\O}{{\cal O}}
\renewcommand{\P}{{\cal P}}
\newcommand{\V}[1][(T, F)]{V_{\K}^{#1}}
\newcommand{\Z}[1][(T, F)]{Z_{\K}^{#1}}
\newcommand{\Atmost}{Atmost_{\K}}

\newcommand{\KA}{{\sf KA}}
\newcommand{\GUS}{{U_{\K}}}

\newcommand{\OBone}[1]{\textsf{OB}_{\mathcal{O}, \it #1}}

\renewcommand{\implies}{\supset}

\newcommand{\mknfmodels}{\models_{{MKNF}}}

\newacronym{mknf}{MKNF}{Minimal Knowledge Negation as Failure}

\newtheorem{theorem}{Theorem}[section]
\newtheorem{proposition}{Proposition}[section]
\newtheorem{example}{Example}
\newtheorem{lemma}{Lemma}[section]
\newtheorem{corollary}[theorem]{Corollary}
\newtheorem{definition}{Definition}[section]

\newcommand{\Definition}[1]{Definition \ref{#1}}

\newcommand{\Proposition}[1]{Proposition \ref{#1}}
\newcommand{\Example}[1]{Example \ref{#1}}
\newcommand{\Lemma}[1]{Lemma \ref{#1}}

\newcommand{\orcid}[1]{\href{https://orcid.org/#1}{\textcolor[HTML]{A6CE39}{\aiOrcid}}}

\begin{document}\sloppy

\title{Unfounded Sets for Disjunctive Hybrid MKNF Knowledge Bases}


\date{}
\maketitle

\begin{abstract}
Combining the closed-world reasoning of answer set programming (ASP) with the open-world reasoning of ontologies broadens the space of applications of reasoners.
Disjunctive hybrid MKNF knowledge bases succinctly extend ASP and in some cases without increasing the complexity of reasoning tasks. However, in many cases, solver development is lagging behind.
As the result, the only known method of solving disjunctive hybrid MKNF knowledge bases is based on guess-and-verify, as formulated by Motik and Rosati in their original work.
A main obstacle is understanding how constraint propagation may be performed by a solver, which, in the context of ASP, centers around the computation of \textit{unfounded atoms}, the atoms that are false given a partial interpretation.
In this work, we build towards improving solvers for hybrid MKNF knowledge bases with disjunctive rules:
We formalize a notion of unfounded sets for these knowledge bases, identify lower complexity bounds, and demonstrate how we might integrate these developments into a solver.
We discuss challenges introduced by ontologies that are not present in the development of solvers for disjunctive logic programs, which warrant some deviations from traditional definitions of unfounded sets.
We compare our work with prior definitions of unfounded sets.
\end{abstract}

\section{Introduction}
Minimal Knowledge and Negation as Failure (MKNF), a modal autoepistemic logic defined by Lifschitz \cite{lifschitz_nonmonotonic_nodate} which extends first-order logic with two modal operators $\bfK$ and $\Not$,
 provides a uniform framework for nonmonotonic reasoning. 
It was later built upon by Motik and Rosati \cite{motik_reconciling_2010} to define hybrid MKNF knowledge bases, where
 rule-based MKNF formulas along with a description logic (DL) knowledge base
intuitively encapsulate the combined semantics of answer set programs and ontologies. 
One argument for using hybrid MKNF is the existence of a proof theory based on guess-and-verify \-- one can enumerate partitions (a term that corresponds to interpretation in first-order logic) and for each one check whether it is an MKNF model.
Such an approach is not efficient enough to be practical in a solver.

To address the above issue, 
Ji et al. \cite{ji_well_founded_2017}
give a definition of unfounded sets and an abstract DPLL-based solver \cite{SMT-JACM}
for normal hybrid \mknf knowledge bases, where rules are constrained to a single atom in the head.

Disjunctive heads in rules are a powerful extension to answer set programming and increase the expressive power of programs in the polynomial complexity hierarchy \cite{eiter_computational_1995}.
In this work, we extend the work of Ji et al. \cite{ji_well_founded_2017} by defining unfounded sets for disjunctive hybrid \mknf knowledge bases and investigate its properties. The problem turns out to be substantially more challenging than the normal case. 
We show the following main results. First, we show that the problem of determining whether an atom is unfounded w.r.t. a given (partial) partition is coNP-hard.
The result is somewhat surprising in that the claim holds even for normal rules under the condition that the entailment relation in the underlying DL is polynomial.
This shows that the polynomial construction of the greatest unfounded set as given by Ji et al. \cite{ji_well_founded_2017} for the normal case is only an approximation. 
Our proof relies on an encoding that takes care of several conditions simultaneously (the hardness in the presence of non-disjunctive rules and the entailment relation under DL is polynomial).
 Then, we formulate a polynomial operator to approximate the greatest unfounded set of disjunctive hybrid \mknf knowledge bases.
 Unlike the conventional definition of unfounded sets for disjunctive logic program \cite{leone_disjunctive_1997}, greatest unfounded sets under our definition exist unconditionally.
 We identify the conditions under which our approximation becomes exact for normal as well as for disjunctive hybrid \mknf knowledge bases.
 These conditions are also the ones under which the coNP-hardness reduces to polynomial complexity for the normal and disjunctive cases respectively, thus these results pinpoint the sources that contribute to the hardness of computing greatest unfounded sets in general.
Finally, based on these results, we formulate a DPLL-based solver, where the computation of unfounded sets becomes a process of
constraint propagation for search space pruning. 

The next section provides preliminaries. Section \ref{section-unfounded-sets} gives the definition of unfounded sets and studies its properties. Section \ref{compute} shows the main technical results concerning the challenges of computing unfounded sets, which lead to a formulation of a DPLL-based solver in Section \ref{dpll}. Section \ref{related} is about related work. The paper is closed by concluding remarks in Section \ref{conclusion}.

\section{Preliminaries}
\label{preliminaries}
Minimal knowledge and negation as failure (MKNF) extends first-order logic with two modal operators, $\bfK$ and  $\Not$, for minimal knowledge and negation as failure respectively.
MKNF formulas are constructed from first-order formulas using these two modal operators for closed-world reasoning.
Intuitively, $\bfK \psi$ asks whether $\psi$ is known w.r.t. a collection of ``possible worlds"  \-- the larger the set, the fewer facts are known \-- while $\Not \psi$ checks whether $\psi$ is not known, based on negation as failure. 
An MKNF structure is a triple $(I, M, N)$ where $I$ is a first-order interpretation and $M$ and $N$ are sets of first-order interpretations.
Operators shared with first-order logic are defined as usual.
The satisfiability under an MKNF structure is defined as:
\begin{itemize}\label{nphard1}
    \item $(I, M, N) \models A$ if A is true in $I$ where A is a ground-atom
    \item $(I, M, N) \models \neg F$ if $(I, M, N) \not\models F$
    \item $(I, M, N) \models F \land G$ if $(I, M, N) \models F$ and $(I, M, N) \models G$
    \item $(I, M, N) \models \exists x, F$ if $(I, M, N) \models F[\alpha / x]$ for some ground atom $\alpha$ \\
    (where $F[\alpha / x]$ is obtained by replacing every occurrence of the variable $x$ with $\alpha$)
    \item $(I, M, N) \models \bfK F$ if $(J, M, N) \models F$ for each $J \in M$
    \item $(I, M, N) \models \Not F$ if $(J, M, N) \not\models F$ for some $J \in N$
\end{itemize}
Other symbols such as $\lor$, $\forall$, and $\supset$ are interpreted in MKNF as they are in first-order logic.
An \textit{\mknf interpretation} $M$ is a set of first-order interpretations; $M$ satisfies a formula $F$, written $M \mknfmodels F$, if $(I, M, M) \models F$ for each $I \in M$.
\begin{definition}
    An \textit{\mknf model} $M$ of a formula $F$ is an \mknf interpretation such that $M \mknfmodels F$ and there does not exist an \mknf interpretation $M' \supset M$ such that  $(I, M', M) \models F$ for each $I \in M'$.
\end{definition}
Following Motik and Rosati \cite{motik_reconciling_2010}, a {\em hybrid MKNF knowledge base} $\K = (\O, \P)$ consists of a decidable description logic (DL) knowledge base $\O$
(typically called an ontology) which is translatable to first-order logic and a set of \mknf rules $\P$.
We denote this translation as $\pi(\O)$.
Rules in $\P$ are of the form:
\begin{equation}
    \bfK a_1, \dots, \bfK a_k \leftarrow  \bfK a_{k+1}, \dots, \bfK a_{m}, \Not a_{m + 1}, \dots, \Not a_{n}
\end{equation}
In the above, $a_1, \dots, a_n$ are function-free first-order atoms of the form $p(t_1, \dots, t_i)$ where $p$ is a predicate and $t_1, \dots, t_i$ are either constants or variables, with $k \geq 1$ and $m,n, i \geq 0$.
A rule $r$ in $\P$ is \textit{DL-safe} if for every variable present in $r$,
there is an occurence of that variable in the rule's positive body inside a predicate that does not occur in $\K$'s description logic.

A hybrid \mknf knowledge base $\K$ is DL-safe if every rule in $\P$ is DL-safe.
A knowledge base that is not DL-safe may not be decidable \cite{motik_reconciling_2010}.
This constraint restricts all variables in $\P$ to names explictly referenced in $\P$.
Throughout this work and without lose of generality we assume that $\P$ is ground, i.e. it does not contain variables.
Let $\pi(\P)$ denote rule set $\P$'s corresponding \mknf formula:
\begin{align*}
    \pi(\P) &= \bigwedge\limits_{r \in \P} \pi(r) \\
    \pi(r) &= \forall \vec{x} \left(  \bigvee\limits_{i = 1}^{k} \bfK a_i \subset \bigwedge\limits_{i = k + 1}^{m} \bfK a_i \land \bigwedge\limits_{i = m + 1}^{n} \Not a_i \right)
\end{align*}
where $\vec{x}$ is the vector of free variables found in $r$.

The semantics of a hybrid \mknf knowledge base $\K$ is obtained by applying both transformations to $\O$ and $\P$ and wrapping $\O$ in a \bfK operator, i.e.
$\pi(\K) = \pi(\P) \land \bfK \pi(\O)$.
We use $\P$, $\O$, and $\K$ in place of $\pi(\P)$, $\pi(\O)$, and $\pi(\K)$ respectively when it is clear from context that the translated variant is intended.
We refer to formulas of the form $\bfK a$ and $\Not a$, where $a$ is a first-order atoms, as \Katoms{} and \Not-atoms respectively, and we refer to them collectively as modal-atoms.
Hybrid \mknf knowledge bases rely on the standard name assumption which requires \mknf interpretations to be Herbrand interpretations with a countably infinite number of additional constants.
In the rest of paper, we may refer to disjunctive hybrid \mknf knowledge bases simply as knowledge bases for abbreviation, or normal knowledge bases if each rule in the knowledge base has exactly one atom in the head.
We outline some definitions and conventions:
For a hybrid \mknf knowledge base $\K = (\O, \P)$, we denote the set of all \Katoms{} in $\P$ with $\KA(\K) = \KA(\P)$ where
\begin{equation}
    \KA(\P) = \{ \bfK a \setcomp \textrm{either $\bfK a$ or $\Not a$ occurs in the head or body of a rule in $\P$} \}
\end{equation}
We use $\bfK(\bodyn(r))$ to denote the set of \Katoms{} converted from \Not-atoms from an \mknf rule $r$'s negative body, i.e. $\bfK(\bodyn(r)) = \{ \bfK a \setcomp \Not a \in \bodyn(r) \}$ and use $\bodyp(r)$ to denote the \Katoms{} from the positive body of the rule $r$.
The \textit{objective knowledge} of a hybrid \mknf knowledge base $\K$ w.r.t. to a set of \Katoms{} $S \subseteq \KA(\K)$, denoted as $\OBone{S}$, is the set of first-order formulas $\{ \pi(\O) \} \union \{ a \setcomp \bfK a \in S \}$.

A \textit{(partial) partition} of $\KA(\K)$ is a nonoverlapping pair $(T, F)$, i.e., $T \intersect F = \emptyset$, where $T$ and $F$ are subsets of $\KA(\K)$.
A partition is \textit{total} if $T \union F = \KA(\K)$.
\textit{A dependable partition} is a partial partition $(T, F)$ with the additional restriction that $\OBone{T} \union \{ \neg b \}$ is consistent for each $\bfK b \in F$ or $\OBone{T}$ is consistent if $F$ is empty. We add this restriction for convenience and note that a partial partition that is not dependable may not be extended to an \mknf model. In practice, a solver includes direct consequences of $\OBone{T}$ in $T$ and it only operates on dependable partitions.
We denote the partition induced by the body of a rule $r$ with $\body(r) = (\bodyp(r), \bfK(\bodyn(r)))$.
A rule body is applicable w.r.t. a partition $(T, F)$ if $\body(r) \sqsubseteq (T, F)$, i.e., if $\bodyp(r) \subseteq T$ and $\bfK(\bodyn(r)) \subseteq F$.
We say that an \mknf interpretation $M$ of $\K$ induces a partition $(T, F)$ if
\begin{equation}
    \bigwedge\limits_{\bfK a \in T} M \mknfmodels \bfK a \land \bigwedge\limits_{\bfK a \in F} M \mknfmodels \neg \bfK a
\end{equation}
If $M$ is an \mknf model of $\K$ and $M$ induces the partition $(T^*, F^*)$, then we say a partition $(T, F) \sqsubseteq (T^*, F^*)$ can be extended to an \mknf model.
Every partition induced by a model is dependable.
Note that for any dependable partition $(T', F')$, every partial partition $(T, F) \sqsubseteq (T', F')$ is dependable.

\section{Unfounded Sets}\label{section-unfounded-sets}

First defined for normal logic programs by van Gelder et al. \cite{van_gelder_well-founded_1991}, unfounded sets encapsulate atoms that must be false w.r.t. a partial interpretation.
Critically, given a partition $(T, F)$ of $\KA(\K)$ that assigns atoms truth or falsity, an unfounded set of a knowledge base $\K$ w.r.t. $(T, F)$ is a set of atoms that must be false if $(T, F)$ can be extended to an \mknf model.

A \textit{head-cut} $R \subseteq \P \times \KA(\K)$ is a set of rule atom pairs such that a rule $r \in \P$ occurs in at most one pair in $R$ and for every pair $(r, h) \in R$ we have $h \in \head(r)$.
We use $\head(R)$ to denote the set $\{ h \setcomp (r, h) \in R  \}$ where $R$ is a head-cut.

\begin{definition}\label{unfounded}\label{dhmknf_unfounded}

    Let $\K = (\O, \P)$ be a disjunctive \mknf knowledge base and $(T, F)$ a partial partition of $\KA(\K)$.
    A set $X$ of \Katoms{} is an \textit{unfounded set} of $\K$ w.r.t. $(T,F)$ if for each \Katom{} $\bfK a \in X$ and each head-cut $R$ such that:

    \begin{enumerate}
        \item\label{uf-r-1} $\head(R) \union \pi(\O) \models a$ (with $\O$, $R$ can derive $\bfK a$), and
        \item\label{uf-r-2} $\head(R) \union \OBone{T} \union \{ \neg b \}$ is consistent for each $\bfK b \in F$ and $ \head(R) \union \OBone{T}$ is consistent if $F$ is empty (the partition $(T \union \head(R), F)$ is dependable),
    \end{enumerate}

    there is a pair $(r, h) \in R$ such that at least one of the following conditions hold:
    \begin{enumerate}[{\em i.}]
        \item\label{uf-conf1} $\bodyp(r) \intersect (F \union X) \not= \emptyset$ ($r$ positively depends on false or unfounded atoms),
        \item\label{uf-conf2} $\bfK(\bodyn(r)) \intersect T \not= \emptyset$ ($r$ negatively depends on true atoms), or
        \item\label{uf-conf3} $\head(r) \intersect T \not= \emptyset$ (rule head is already satisfied)
    \end{enumerate}
A $\bfK$-atom in an unfounded set is called an {\em unfounded atom}.
\end{definition}

We illustrate some general characteristics of this definition with the following example. 
\begin{example}
    Let $\K = (\O, \P)$ where
    \begin{align*}
        \O = \{ & (a \supset a' ) \wedge ( b \supset b') \wedge \neg f \}\\
        \P = \{ & \bfK f \leftarrow b ;  \\
                & \bfK a \leftarrow \Not b ; \\
                & \bfK a, \bfK b, \bfK c \leftarrow;  \\
                & \bfK a' \leftarrow \bfK a'; ~\bfK b' \leftarrow \bfK b'\}
    \end{align*}
    Let $(T, F)$ be the dependable partition $(\{ \bfK b \}, \emptyset)$.
    The \Katom{} $\bfK f$ is an unfounded atom w.r.t. $(T, F)$ because $\bfK f$ creates an inconsistency in $\O$.
    $\bfK a$ is an unfounded atom because the only way of deriving $\bfK a$ relies on $\neg \bfK b$ which contradicts $T$.
    The \Katom{} $\bfK a'$ is unfounded because $\bfK a$ is unfounded and $\bfK b'$ is not unfounded because $\bfK b$ is in $T$.
    Lastly, $\bfK c$ is an unfounded atom because the only rule that can derive $\bfK c$ has another head-atom ($\bfK b$) in $T$.
\end{example}

An unfounded set $X$ w.r.t. a dependable partition $(T, F)$ is a set of \Katoms{} that must be false should $(T, F)$ be extended to an \mknf model.
A head-cut $R$ is a set of rules that may be used in conjunction with $(T, F)$ to derive a \Katom{}.
A \Katom{} is unfounded only if every head-cut that can derive it has a pair in it that meets one of the conditions \textit{\ref{uf-conf1}} through \textit{\ref{uf-conf3}}.
Note that if $(T, F)$ is dependable, then it is impossible to derive a \Katom{} in $F$ without violating condition \ref{uf-r-2} because an empty head-cut can be used to derive any \Katom{} in $T$. 
We demonstrate this property in the following example.
\begin{example}
    Let $\K = (\O, \P)$ where
    \begin{align*}
        \O & = \{ a \supset b \} \textrm{, and }                            \\
        \P & = \{ \bfK a \leftarrow \Not b; ~\bfK b \leftarrow \Not a \}
    \end{align*}
    The dependable partition $(\{ \bfK b \}, \{ \bfK a \})$ is the only total dependable partition induced by an \mknf model of $\K$.
    Suppose we have the dependable partition $(T, F) = (\{ \bfK a \}, \emptyset)$.
    Note that $(T, F)$ cannot be extended to an \mknf model. 
    Neither $\bfK a$ nor $\bfK b$ is an unfounded atoms w.r.t. $(T, F)$: when $R = \emptyset$ we have
$\head(R) \union \OBone{T} \models a$ and $\head(R) \union \OBone{T} \models b$.
Let $(T, F) = (\emptyset, \{ \bfK a \})$ be a different dependable partition.
The \Katom{} $\bfK a$ is an unfounded atom w.r.t. $(T, F)$.
The only head-cut that can derive $\bfK a$ is the set $R = \{ (\bfK a \leftarrow \Not b) \}$, however, $\head(R) \union \OBone{T} \union \{ \neg a \}$ can be rewritten as $\{ a \} \union \OBone{T} \union \{ \neg a \}$ which is inconsistent.
\end{example}

Under \Definition{unfounded}, atoms in $T$ cannot be unfounded.
We formally establish that no \Katom{} in $T$ can be an element in an unfounded set in the following lemma.

\begin{lemma}[$T$ is disjoint from any unfounded set]\label{T_disjoint_unfounded}
    Let $U$ be an unfounded set of a disjunctive knowledge base $\K$ w.r.t. a dependable partition $(T, F)$ of $\KA(\K)$.
    We have $T \intersect U = \emptyset$.
\end{lemma}

\begin{proof}
    Assume for the sake of contradiction that $U \intersect T \not= \emptyset$, and let $\bfK a \in U \intersect T$.
    Because $U$ is an unfounded set w.r.t. $(T, F)$ we have for every head-cut $R$ such that $\head(R) \union \OBone{T} \models a$, $\OBone{T} \union \{ \neg b \}$ is consistent for each \Katom{} $\bfK b \in F$, and $\OBone{T}$ is consistent, that there is a pair $(r, h) \in R$ such one of the conditions \textit{\ref{uf-conf1}}, \textit{\ref{uf-conf2}}, or \textit{\ref{uf-conf3}} is satisfied.
    Let $R = \emptyset$.
    We have $\head(R) \union \OBone{T} \models a$ because $\bfK a \in T$.
    Because $(T, F)$ is dependable, $\OBone{T} \union \{ \neg b \}$ is consistent for each \Katom{} $\bfK b \in F$, and $\head(R) \union \OBone{T}$ is consistent.
    However, there does not exist a pair $(r, h) \in R$ because $R$ is empty, a contradiction.
\end{proof}

The property demonstrated in \Lemma{T_disjoint_unfounded} is inherited from the definition of unfounded sets for normal hybrid \mknf knowledge bases \cite{ji_well_founded_2017}.
This is quite different from the definition of unfounded sets for disjunctive logic programs:
Leone et al. \cite{leone_disjunctive_1997} refer to (partial) partitions (called {\em interpretations} in their context) where no atom in $T$ is unfounded (under their own definition of unfounded sets) as {\em unfounded-free}.
In some respects, unfounded sets under Leone et al. \cite{leone_disjunctive_1997} can doubt the truth of \Katoms{} in $T$.
Since unfounded atoms are assumed to be false, an unfounded set w.r.t. $(T, F)$ that shares \Katoms{} with $T$ is proof that $(T, F)$ cannot be extended to a model.
As shown in \Lemma{T_disjoint_unfounded}, \Definition{unfounded} lacks this property.
We illustrate this difference in the following example.

 \begin{example}\label{eg-empty-ontology}
     Let $\K = (\emptyset, \P)$ where $\P = \{ \bfK a, \bfK b \leftarrow  \}$ and construct the dependable partition $(T, F) = (\{ \bfK a, \bfK b \}, \emptyset)$.
     Under Leone et al.'s definition, both $\{ \bfK a \}$ and $\{ \bfK b \}$ are unfounded sets w.r.t. $(T, F)$, however, the set $\{ \bfK a, \bfK b \}$ is not an unfounded set w.r.t. $(T, F)$.
     Under \Definition{unfounded}, none of the three aforementioned sets are unfounded sets w.r.t. $(T, F)$ due to \Lemma{T_disjoint_unfounded}. 
 \end{example}

Leone et al. show that partial partitions that have the unfounded-free property and satisfy every rule in $\P$ are precisely the partial partitions that can be extended to stable models \cite{leone_disjunctive_1997}.
In the example above, the dependable partition $$(T, F) =(\{ \bfK a, \bfK b \}, \emptyset)$$ cannot be extended to an \mknf model and neither $\bfK a$ nor $\bfK b$ is an unfounded atom w.r.t. $(T, F)$.
This indicates that unfounded sets under \Definition{unfounded} cannot be used to determine whether a partition can be extended to an \mknf model in the same way as Leone et al.
We demonstrate that this is the case even for a normal knowledge base with an empty ontology.

\begin{example}
    Let $\K = (\emptyset, \P)$ where $\P = \{ \bfK a \leftarrow \Not a \}$.
    Note that $\K$ does not have an \mknf model.
    The two possible total partitions are $(T_1, F_1) = (\emptyset, \{ \bfK a \})$ and $(T_2, F_2) = (\{ \bfK a \}, \emptyset)$.
    Under both \Definition{unfounded} and Leone et al.'s definition of unfounded sets, the only unfounded set w.r.t. $(T_1, F_1)$ is $\emptyset$.
    Like Leone et al., we can determine that $(T_1, F_1)$ is not an \mknf model of $\K$ because there is a rule $r \in \P$ such that $\body(r) \sqsubseteq (T_1, F_1)$ and $\head(r) \intersect T_1 = \emptyset$.
    Under Leone et al.'s definition, the set $\{ \bfK a \}$ is an unfounded set of $\P$ w.r.t. $(T_2, F_2)$, however, $\{ \bfK a \}$ is not an unfounded set of $\K$ w.r.t. $(T_2, F_2)$ under \Definition{unfounded}.
    We cannot use \Definition{unfounded} to conclude that there is not \mknf model that induces $(T_2, F_2)$.  
\end{example}

The above example demonstrates a limitation that prevents unfounded sets from being used as a mechanism for \mknf model checking.
This limitation is also present in the unfounded sets defined by Ji et al. \cite{ji_well_founded_2017}, however, it does not inhibit unfounded sets from being useful in a solver.
Following Ji et al \cite{ji_well_founded_2017} and Leone et al. \cite{leone_disjunctive_1997}, we show that unfounded sets in \Definition{unfounded} are closed under union.
The property that all dependable partitions are unfounded-free (\Lemma{T_disjoint_unfounded}) removes the need for an additional restriction on partitions as is needed for disjunctive logic programs \cite{leone_disjunctive_1997}.

The unfoundedness of some \Katoms{} is dependant on the unfoundedness of other \Katoms{} (condition \textit{\ref{uf-conf1}} of \Definition{dhmknf_unfounded}), thus new unfounded sets can be constructed by adding certain \Katoms{} to smaller unfounded sets.
Condition \textit{\ref{uf-conf3}} of \Definition{unfounded} ($\head(r) \intersect T \not=\emptyset$) does not depend on the unfounded set $X$ like it does in Leone et al.'s definition (in this context, $(\head(r) \setminus X) \intersect T \not=\emptyset$).
Applying \Lemma{T_disjoint_unfounded}, $(\head(r) \setminus X) \intersect T \not=\emptyset$ can be rewritten as $\head(r) \intersect T \not=\emptyset$.
This results in unfounded sets being closed under union in general.
We demonstrate this property formally in the following proposition.
\begin{proposition}[Existence of a greatest unfounded set]\label{gus}
For a disjunctive hybrid \mknf knowledge base $\K = (\O, \P)$ and a partial partition $(T, F)$ of $\KA(\K)$, there exists a \textit{greatest unfounded set} $\GUS(T, F)$ such that $X \subseteq \GUS(T, F)$ for every unfounded set $X$ of $\K$ w.r.t. $(T, F)$.
\end{proposition}

\begin{proof}
We show that unfounded sets are closed under union and the existence of a greatest unfounded set directly follows.
Let $X_a$ and $X_b$ be unfounded sets of $\K$ w.r.t. a partial partition $(T, F)$ of $\KA(\K)$.
We show that the set $X_c = X_a \union X_b$ is an unfounded set of $\K$ w.r.t. $(T, F)$.
If $(T, F)$ is not dependable, then every set $X \subseteq \KA(\K)$ is an unfounded set of $\K$ w.r.t. $(T, F)$ including $X_c$.
Assume that $(T, F)$ is dependable and for the sake of contradiction, assume $X_c$ is not an unfounded set.
For some \Katom{} $\bfK a \in X_c$ we have a head-cut $R$ s.t. conditions \ref{uf-r-1} ($\head(R) \union \OBone{T} \models a$) and \ref{uf-r-2} ($\head(R) \union \OBone{T} \union \{ \neg b \}$ is consistent for each $\bfK b \in F$ or $\head(R) \union \OBone{T}$ is consistent if $F$ is empty) hold.
In this head-cut, there is a pair $(r, a)$ such that none of the
conditions \textit{\ref{uf-conf1}} ($\bodyp(r) \intersect (X_c \union F) \not= \emptyset$), \textit{\ref{uf-conf2}} ($\bodyn(r) \intersect T \not= \emptyset$), or \textit{\ref{uf-conf3}} ($\head(r) \intersect T \not= \emptyset$) hold.
For simplicity, assume $\bfK a \in X_a$ (proof is identical if $\bfK a \in X_b$).
If $\bodyp(r) \intersect (X_a \union F) \not= \emptyset$ then we have $\bodyp(r) \intersect (X_c \union F) \not= \emptyset$ and it follows that $X_c$ is an unfounded set.
\end{proof}

This property is a natural result of \Lemma{T_disjoint_unfounded} and differs from
Leone et al.'s unfounded sets are closed under union only if $(T, F)$ is unfounded-free.

A solver can use any unfounded set to extend a dependable partition's false atoms without affecting the models it finds.
We now relate unfounded sets to \mknf models.

\begin{proposition}\label{notunfounded_is_T}
Let $(T^*, F^*)$ be the partition induced by an \mknf model of a disjunctive hybrid \mknf knowledge base $\K$. For any dependable partition $(T, F) \sqsubseteq (T^*, F^*)$, $\GUS(T, F) \intersect T^* = \emptyset$.
\end{proposition}

\begin{proof}
Note that $(T^*, F^*)$ is total and dependable.
Let $(T, F) \sqsubseteq (T^*, F^*)$ and $U$ be an unfounded set of $\K$ w.r.t. $(T, F)$.
Let $U$ be an unfounded set of $\K$ w.r.t. $(T, F)$.
We show that $U \intersect T^* = \emptyset$ and it follows that $\GUS(T, F) \intersect T^* = \emptyset$.
Assume for the sake of contradiction that $U \intersect T^* \not= \emptyset$; have $B =  U \intersect T^*$, then construct an \mknf interpretation $M'$ such that 
\begin{equation}
    M' = \{ I \setcomp I \models \OBone{T} \textrm{ and  } I \models t \textrm{ for each $\bfK t \in T^* \setminus B$ } \} \}
\end{equation}
The dependable partition induced by $M'$ is $(T^* \setminus B, F^* \union B)$.
For each $b \in B$, $\OBone{T} \not\models b$, thus $M' \supset M$. 
We derive a contradiction by showing $U$ is not an unfounded set of $\K$ w.r.t. $(T, F)$.
By construction, $(I, M', M) \mknfmodels \pi(\O)$ for each $I \in M'$.
If $\forall I \in M', (I, M', M) \mknfmodels \pi(\P)$, then $M$ is not a model, a contradiction.
Using $(T^* \setminus B, F)$ to denote the partition used to test each rule $r \in \P$, observe that
if $r$ is not satisfied w.r.t. $(T^* \setminus B, F)$ it of the form $\body(r) \sqsubseteq (T^* \setminus B, F)$, $\head(r) \intersect T^* \not= \emptyset$, and $\head(r) \intersect (T^* \setminus B) = \emptyset$.
$r$ is a rule whose body is satisfied by $(T^* \setminus B, F)$ but all true atoms in its head come from $B$.
Let $R = \{ (r, h) \}$ where $b$ is some atom from $\head(r) \intersect B$.
Conditions 1 and 2 of \Definition{unfounded} are met for $R$ to test if $U$ is an unfounded set of $\K$ w.r.t. $(T, F)$.
We show that none of the conditions $i$ through $iii$ are met by $R$ showing that $U$ is not an unfounded set w.r.t. $(T, F)$, a contradiction.
First, $\bodyp(r) \subseteq T^* \setminus B$ gives us $\bodyp(r) \intersect (F \union U) = \emptyset$.
From $\bfK(\bodyn(r)) \subseteq F$, we derive $\bfK(\bodyn(r)) \intersect T \not= \emptyset$.
Finally, using $\head(r) \intersect T^* \subseteq B$ and $B \intersect T = \emptyset$ (\Lemma{T_disjoint_unfounded}), we conclude $\head(r) \intersect T = \emptyset$.
We have shown $U \intersect T^* = \emptyset$, as desired. 
\end{proof}

We've shown that if a dependable partition $(T, F)$ can be extended to an \mknf model, no unfounded set of $\K$ w.r.t. $(T, F)$ may overlap with the true atoms in the model. It follows directly from \Proposition{notunfounded_is_T} that the following analogous property holds for atoms in $F$.
\begin{corollary}\label{unfounded_is_F}
    Let $(T^*, F^*)$ be the partition induced by an \MKNF model $M$ of a disjunctive hybrid \mknf knowledge base $\K$.
    Then, for any dependable partition $(T, F) \sqsubseteq (T^*, F^*)$, $M \mknfmodels \neg \bfK u$ for all $u \in \GUS(T, F)$.
\end{corollary}

With these properties, we've shown that unfounded sets can be used to extend a partition without missing any models, i.e., if $(T, F)$ can be extended to an \mknf model $M$ then $(T, F \union U)$ can be extended to the same model $M$ for any unfounded set $U$ w.r.t. $(T, F)$.

\section{Computing Unfounded Sets}
\label{compute}

Due to the inconsistencies that can arise in connection with $\O$, computing the greatest unfounded set w.r.t. a partial partition is intractable in general.

\begin{example}
    Let $\K = (\O, \P)$ where 
        $\O = \neg (a \land b)$
    and
    \begin{align*}
        \P = \{  \bfK a \leftarrow \Not b;~
            \bfK b \leftarrow \Not a;~
            \bfK c \leftarrow \bfK c \}
    \end{align*}
    Under \Definition{unfounded}, $\bfK c$ is an unfounded atom w.r.t. $(\emptyset, \emptyset)$, however, with the $V_{\K}^{(\emptyset, \emptyset)}$ operator defined by Ji et al. \cite{ji_well_founded_2017} we have $\textbf{lfp}(V_{\K}^{(\emptyset, \emptyset)}) = \KA(\K)$ which misses $\bfK c$ as an unfounded atom.\footnote{This is because in the least fixed point computations of the $\V$ operator, a default negation $\Not q$ is true if $\bfK q$ is not known to be true, and as such, both $\bfK a$ and $\bfK b$ are derived in the first iteration which leads to inconsistency with $\O$.}
    It's clear that a similar operator for disjunctive knowledge bases would have the same limitation.
\end{example}

In the following, we first give a formal proof of intractability and then we construct an operator for hybrid MKNF knowledge bases with disjunctive rules that adopts the same approximation technique used by Ji et al. in their $\V$ operator \cite{ji_well_founded_2017} for hybrid MKNF knowledge bases with normal rules.

We now show that deciding whether an atom of a normal hybrid \mknf knowledge base is unfounded is 

coNP-hard by comparing the head-cuts that need to be considered to determine unfoundedness with the SAT assignments that need to be considered to determine the satisfiability of a 3SAT problem.
\begin{proposition}\label{nphard1}
Let $\K = (\O, \P)$ be a normal hybrid \mknf knowledge base such that the entailment relation $\OBone{S} \models a$ can be checked in polynomial time for any set $S \subseteq \KA(\K)$ and for any \Katom{} $\bfK a \in \KA(\K)$.
Determining whether a \Katom{} $\bfK a \in \KA(\K)$ is an unfounded atom of $\K$ w.r.t. a dependable partition $(T, F)$ of $\KA(\K)$ is coNP-hard.
\end{proposition}

\begin{proof}
We show that the described problem is coNP-hard.
The 3SAT problem is well known to be NP-complete \cite{sipser_introduction_1996}.
Let $SAT$ be an instance of 3SAT in conjunctive normal form such that $CLAUSE = \{ c_1, c_2, \dots, c_n \}$ is the set of clauses in $SAT$ and $VAR =  \{ v_1, v_2, \dots, v_n \}$ is the set of variables in $SAT$.
Determining whether $SAT$ is unsatisfiable is coNP-hard.
We construct a normal hybrid \mknf knowledge base $\K = (\O, \P)$ s.t.
\begin{multline}
\begin{aligned}
\O = & \{ v^{u}_{i} \lxor v^{f}_{i} \lxor v^{t}_{i} \setcomp \textrm{ for each $v_i \in VAR$ where $\lxor$ is exclusive-or } \} \union \\
     & \{ ( \bigwedge\limits_{v_i \in VAR} \neg v^{u}_i \iff total ) \textrm{, } total \supset sat \} \union \\
     & \{ clause_i \lor \neg total \setcomp \textrm{ for each clause $c_i \in CLAUSE$ where $clause_i$ is a formula}\\
     & \textrm{ obtained by replacing all occurences of $v_i$ and $\neg v_i$ in $c_i$ with $v^{t}_{i}$ and $v^{f}_{i}$ respectively} \}
\end{aligned}
\end{multline}
and
\begin{multline}
\begin{aligned}
\P =  \{ \bfK sat \leftarrow \bfK sat  \} \union \bigcup \{ \{ (\bfK v_i^{t} \leftarrow \Not v_i^{f}), (\bfK v_i^{f} \leftarrow \Not v_i^{t}) \} \setcomp v_i \in VAR \}
\end{aligned}
\end{multline}
Note that the rule $\bfK sat \leftarrow \bfK sat$ is only required to ensure that $\bfK sat$ is in $\KA(\K)$.
The time to construct the above knowledge base is linear in the number of clauses and variables in $SAT$.
The first set of formulas in $\O$ require exactly one of $v_i^u$, $v_i^f$, or $v_i^t$ to be true.
This constraint is analogous to a three-valued assignment for $SAT$ where a variable $v_i \in VAR$ is unassigned if $v_i^u$ is true, assigned false if $v_i^f$ is true, and assigned true if $v_i^t$ is true.
The second set in $\O$ ensures that the atom $total$ is true if and only if no variable is unassigned.
Finally, the third set of formulas ensure that $\pi(\O)$ is inconsistent if the assignment is total and a clause in $SAT$ is not satisfied.
We show that (1) For any \Katom{} $\bfK a$ and set of \Katoms{} $S$, the entailment relation $\OBone{S} \models a$ is computable in polynomial time and (2) that $\bfK sat$ is an unfounded atom of $\K$ w.r.t. $(\emptyset, \emptyset)$ if and only if $SAT$ is unsatisfiable.

(1) We call a set of \Katoms{} $S$ total if it contains either $\bfK v_i^{t}$ or $\bfK v_i^{f}$ for each variable $v_i \in VAR$.
Note that for a variable $v_i \in VAR$, the set $\KA(\K)$ only contains $\bfK v_i^{t}$ and $\bfK v_i^{f}$; It does not contain $\bfK v_i^{u}$.
Let $S \subseteq \KA(\K)$.
We show that we can, in polynomial time, determine whether $S \union \pi(\O)$ is consistent. 
We split cases where $S$ is total where it is not.
First, assume $S$ is not total:
For some variable $v_i \in VAR$, neither $\bfK v_i^{t}$ nor $\bfK v_i^{f}$ is in $S$.
By fixing $v_i^{u}$ to be true in a consistent first-order interpretation $I$ of $S \union \pi(\O)$, we ensure the atom $total$ is false.
If the atom $total$ is false, we can determine whether $\OBone{S}$ is consistent in polynomial time because we only need to consider the first two sets of formulas in $\O$.
If $S$ is total, we can, in polynomial time, verify that $S \union \pi(\O)$ is consistent by checking that only one of $v_i^{t}$ or $v_i^{f}$ is present in $S$ and that every clause $clause_i$ is satisfied.
After determining whether $S \union \pi(\O)$ is consistent, we can quickly check the relations $\OBone{S} \models v_i^{t}$ and $\OBone{S} \models v_i^{f}$ for any variable $v_i \in VAR$:
Assuming $S \union \pi(\O)$ is consistent, the entailment relation $\OBone{S} \models v_i^{t}$ (resp. $\OBone{S} \models v_i^{f}$) holds if and only if $\bfK v_i^{t} \in S$ (resp. $\bfK v_i^{f} \in S$).
When $S \union \pi(\O)$ is consistent, the entailment relation $\OBone{S} \models total$ holds if and only if $S$ is total.
Finally, we have $\OBone{S} \models sat$ if and only if $\bfK sat \in S$ or $\OBone{S} \models total$.
If $S \union \pi(\O)$ is inconsistent, the entailment relation $\OBone{S} \models \bfK a$ holds vacuously where $\bfK a \in \KA(\K)$.

(2) When determining whether the \Katom{} $\bfK sat$ is unfounded w.r.t. $(\emptyset, \emptyset)$, we must consider each way to select a head-cut $R$.
We show that there is a correspondence between the head-cuts that can disprove the unfoundedness of $\bfK sat$ w.r.t. $(\emptyset, \emptyset)$ and total sat assignments for $SAT$.
Let $X = \{ \bfK sat \}$ be a set that is possibly unfounded w.r.t. $(\emptyset, \emptyset)$.
Observe that a larger unfounded set $X' \supset X$ w.r.t. $(\emptyset, \emptyset)$ cannot exist unless $X$ is an unfounded set w.r.t. $(\emptyset, \emptyset)$.
A head-cut $R$ cannot be used to disprove the unfoundedness of $\bfK sat$ if either condition 1 or 2 of \Definition{unfounded} do not hold.
Before creating a mapping between head-cuts and sat assignments for $SAT$, we exclude head-cuts that cannot be used to disprove the unfoundedness of $\bfK sat$, i.e., conditions 1 and 2 of \Definition{unfounded} are met and \textit{i}, \textit{ii}, and \textit{iii} do not hold. 
Firstly, we exclude head-cuts that contain the pair $(r, sat)$ because $\bodyp(r) \intersect X \not= \emptyset$.
We further exclude any head-cut $R$ containing a pair of pairs $(r_0, v_i^{t})$ and $(r_1, v_i^{f})$\footnote{Due to the uniqueness of the second component in such a pair, there should be no confusion about which rule the first component refers to} because $\head(R) \union \OBone{\emptyset}$ is inconsistent.
Thirdly, we exclude any head-cuts that do not contain either $(r_0, v_i^{t})$ or $(r_1, v_i^{f})$ for each variable $v_i \in VAR$ noting that if such a head-cut $R$ also meets the previous two conditions we have $\head(R) \union \OBone{\emptyset} \not\models sat$ (See (1) for details).
The remaining head-cuts have a one to one correspondence with total assignments for $SAT$: if a head-cut contains a pair with $v_i^{t}$ (resp. $v_i^{f}$) the corresponding assignment for $SAT$ assigns $v_i$ to be true (resp. false).
We have for every such head-cut $R$ that $\head(R) \union \OBone{\emptyset} \models sat$ and that for every pair in $(r, h) \in R$ we have $\head(r) \intersect T = \emptyset$, $\bodyp(r) \intersect (F \union X) = \emptyset$, and $\bodyn(r) \intersect T = \emptyset$.
If $\head(R) \union \OBone{\emptyset}$ is consistent, then every clause is satisfied by the corresponding sat assignment, otherwise, the inconsistency is caused by an unsatisfied clause $\neg clause_i$, thus the assignment does not satisfy $SAT$.
If no such head-cut $R$ exists such that $\head(R) \union \OBone{\emptyset}$ is consistent, then $\bfK sat$ is unfounded w.r.t. $(\emptyset, \emptyset)$ and $SAT$ is unsatisfiable.
Conversely, if $SAT$ is unsatisfiable, a head-cut $R$ such that $\head(R) \union \OBone{\emptyset}$ is consistent and $\head(R) \union \OBone{\emptyset} \models sat$ does not exist, thus $\{ \bfK sat \}$ is an unfounded set w.r.t. $(\emptyset, \emptyset)$.
We've shown that deciding whether an \Katom{} is unfounded is coNP-hard.
\end{proof}

It follows that computing the greatest unfounded set of a disjunctive hybrid \mknf knowledge base is coNP-hard.
Since we are unlikely to find a way to compute $\GUS(T, F)$ in polynomial time, we are motivated to construct a polynomial operator that computes an approximation (a subset) of the greatest unfounded set.
We define a family of operators $Z_{\K}^{(T, F)}$ where each operator induced by a dependable partition $(T, F)$ computes an approximation of the greatest unfounded set of $\K$ w.r.t. $(T, F)$
\begin{multline}\label{Z}
\begin{aligned}
    Z_{\K}^{(T, F)}: 2^{\KA(\K)} & \rightarrow 2^{\KA(\K)} \\
    X \mapsto 
    & T \union \{ \bfK a \setcomp \OBone{X} \models a \textrm{ for each } \bfK a \in \KA(\K)\} \textrm{ }\union \\ 
    \{ & \bfK a \setcomp \exists r \in \P \textrm{ with } \bfK a \in \head(r) \textrm{ s.t. } \\
    & \bodyp(r) \subseteq X \textrm{ }\land \\
    & \bodyp(r) \intersect F = \emptyset \textrm{ }\land \\
    & \bfK(\bodyn(r)) \intersect T = \emptyset \textrm{ }\land \\
    & \head(r) \intersect T = \emptyset \textrm{ }\land \\
    & \{ a, \neg b  \} \union \OBone{T} \textrm{ is consistent for each } \bfK b \in F \}
\end{aligned}
\end{multline}
This operator is the direct result of combining the $V_{\K}^{(T, F)}$ operator for normal hybrid \mknf \cite{ji_well_founded_2017} with the $\Phi$ operator for disjunctive logic programs \cite{leone_disjunctive_1997}.
It is easy to see that the $Z_{\K}^{(T, F)}$ operator is monotonic, and let us use $\Atmost(T, F)$ to denote its least fixed point.
This operator computes a subset of $\KA(\K) \setminus \GUS(T, F)$.
Firstly, if $\Atmost(T, F) \union \pi(\O)$ is inconsistent, we have $\KA(\K) \setminus \Atmost(T, F) = \emptyset$; A compromise to keep the operator computable in polynomial time.
\par
To determine whether an atom is unfounded when there are disjunctive rules, we must consider an exponential number of head-cuts.
The $Z_{\K}^{(T, F)}$ operator instead considers the heads of rules all at once and this can result in $\Atmost(T, F)$ missing some unfounded atoms even if $\Atmost(T, F) \union \pi(\O)$ is consistent.
\begin{example}\label{gus_operator_eg_1}
Let $\K \! =\! (\O, \P)$ be a disjunctive hybrid \mknf knowledge base where $\P = \{ \bfK a, \bfK b \leftarrow;~ \bfK c \leftarrow \bfK c \}$
and $\O = (a \land b) \supset c$.
We have that $\{ c \}$ is an unfounded set of $\K$ w.r.t. $(\emptyset, \emptyset)$. However, $\Atmost(T, F)$ is $\{ a, b, c \}$ and $\KA(\K) \setminus \{ a, b, c \} \not= \GUS(\emptyset, \emptyset)$.
\end{example}

We intend to identify the class of knowledge bases for which the $Z_{\K}^{(T, F)}$ operator does not miss unfounded atoms as a result of disjunctive heads. First we define a \textit{weak head-cut} to be a set of rule atom pairs $R^w$ such that $R^w \subseteq \P \times \KA(\K)$ and $h \in \head(r)$ for each pair $(r, h) \in R$.
Note that this definition is identical to the definition of head-cuts without the constraint that a rule can appear in at most one pair in $R^w$;
within a weak head-cut, there may be two pairs $(r, h_0)$ and $(r, h_1)$ such that $h_0 \not= h_1$. 
In the following, we define a property that captures a subset of knowledge bases where $\Atmost(T, F) $ computes $\GUS(T, F)$ if $\Atmost(T, F) \union \pi(\O)$ is consistent.
\begin{definition}
    A hybrid \mknf knowledge base $\K = (\O, \P)$ is \textit{head-independent} w.r.t. a dependable partition $(T, F)$ if for every \Katom{} $\bfK a \in \KA(\K)$ and every weak head-cut $R^{w}$ such that $\head(R) \union \OBone{T} \models a$, there exists a head-cut $R$ such that $R \subseteq R^{w}$ and $\head(R) \union \OBone{T} \models a$.
\end{definition}
Head-independence means that we cannot derive atoms that we would not be able to derive using only a single atom from each rule head by using multiple atoms in the head of a rule in conjunction with the ontology.
The head-independence property is violated by the knowledge base in \Example{gus_operator_eg_1} and it ensures that $\Atmost(T, F) \not= \GUS(T, F)$.
Were we to alter the knowledge base in \Example{gus_operator_eg_1} such that the rule $\bfK a, \bfK b \leftarrow$ were changed to the pair of rules $\bfK a \leftarrow \Not b$ and $\bfK b \leftarrow \Not a$ then $\K$ would have head-independence.
We show formally that for a head-independent knowledge base $\K$, the $\Z$ operator computes the greatest unfounded set w.r.t. $(T, F)$ if $\Atmost(T, F) \union \pi(\O)$ is consistent.

\begin{proposition}\label{computes-approx}
    If $\K$ is head-independent w.r.t. a dependable partition $(T, F)$ and $\Atmost(T, F) \union \pi(\O)$ is consistent, then $\GUS(T, F) = \KA(\K) \setminus \Atmost(T, F)$.
\end{proposition}
\begin{proof}
    First we show (1) that no \Katom{} computed by $\Atmost(T, F)$ is unfounded w.r.t. $(T, F)$ and then we show (2) that every atom that is not unfounded w.r.t. $(T, F)$ is computed by $\Atmost(T, F)$.  
    
    (1) We first show no \Katom{} in $Z^{(T, F)}_{\K}(\emptyset)$ is unfounded.
    Let $\bfK a \in Z^{(T, F)}_{\K}(\emptyset)$.
    Construct a weak head-cut $R^{w}$ that contains a pair $(r, h)$ for each head \Katom{} $\bfK h \in \head(r)$ and rule $r \in \P$ where $\bodyp(r) \subseteq \emptyset$, $\bfK(\bodyn(r)) \intersect T = \emptyset$, and $\head(r) \intersect T = \emptyset$.
    The weak head-cut $R^{w}$ contains every rule that was applied in the computation of $Z^{(T, F)}_{\K}(\emptyset)$.
    We have $\head(R^{w}) \union \OBone{T} \models a$.
    Applying the head-independence condition, we obtain a head-cut $R$ such that $R \subseteq R^{w}$ and $\head(R) \union \OBone{T} \models a$.
    For every pair $(r, h) \in R$, $\bodyp(r) \subseteq T$, $\bfK(\bodyn(r)) \intersect T = \emptyset$, and $\head(r) \intersect T$.
    The head-cut $R$ shows that $\bfK a$ is not an unfounded atom w.r.t. $(T, F)$, thus it is not a member of any unfounded set.
    We show that no atom computed by a successive application of $Z^{(T, F)}_{\K}$, e.g. $Z^{(T, F)}_{\K}(Z^{(T, F)}_{\K}(\emptyset))$, is unfounded w.r.t. $(T, F)$.
    Let $Z_i$ be result of applying the $Z^{(T, F)}_{\K}$ operator $i$ times where $Z_0 = \emptyset$.
    We assume that no atom in $Z_i$ is unfounded w.r.t. $(T, F)$ and show the same for $Z_{i+1}$.
    Construct a weak head-cut $R^{w}$ that contains a pair $(r, h)$ for each head a \Katom{} $\bfK h \in \head(r)$ and rule $r \in \P$ where $\bodyp(r) \subseteq Z_i$, $\bodyn(r) \intersect T = \emptyset$, and $\head(r) \intersect T = \emptyset$.
    Let $\bfK a \in Z^{(T, F)}_{\K}(Z_i)$.
    We have $\head(R^{w}) \union \OBone{T} \models a$.
    Applying the head-independence condition, we obtain a head-cut $R$ such that $R \subseteq R^{w}$ and $\head(R) \union \OBone{T} \models a$.
    Now we have for each pair $(r, h) \in R$, $\bodyp(r) \subseteq Z_i$.
    Knowing that no \Katom{} in $Z_i$ is a member of an unfounded set, we conclude that $a$ is not an unfounded atom w.r.t. $(T, F)$.

    (2) We show that if a \Katom{} $\bfK a$ is not computed by $\Atmost(T, F)$ and it is not an unfounded atom w.r.t. $(T, F)$ we can derive a contradiction.
    Let $U = \KA(\K) \setminus \Atmost(T, F)$.
    Let $\bfK a \in U$ be an \Katom{} such that there exists a head-cut $R$ where $\head(R) \union \OBone{T} \models a$, $\head(R) \union \OBone{T}$ is consistent and $\head(R) \union \OBone{T} \union \{ \neg b \}$ is consistent for each $\bfK b \in F$ and for each pair $(r, h) \in R$, $\head(r) \intersect T = \emptyset$ and $\bodyn(r) \intersect T = \emptyset$.
    If for each pair $(r, h) \in R$ we have $\bodyp(r) \not\subseteq \Atmost(T, F)$ then $\bfK a \in \Atmost(T, F)$, otherwise $U$ is an unfounded set w.r.t. $(T, F)$.
    Both cases contradict the initial assumptions.
\end{proof}

For normal knowledge bases, i.e., where each rule contains only a single head-atom, the head-independence condition is satisfied automatically.
If a knowledge base $\K$ is not head-independent, the $\Z$ operator computes a subset of $\GUS(T, F)$.
Therefore, for a normal knowledge base and dependable partition $(T, F)$ s.t. $\Atmost(T, F) \union \pi(\O)$ is consistent, we have $\GUS(T, F) = \KA(\K) \setminus \Atmost(T, F)$.
The following corollary follows directly from \Proposition{computes-approx}.
\begin{corollary}
    If a knowledge base $\K$ is head-independent w.r.t. a dependable partition $(T, F)$ and $\Atmost(T, F) \union \pi(\O)$ is consistent, then the greatest unfounded set of $\K$ w.r.t. $(T, F)$ is computable in polynomial time.
\end{corollary}

We have shown that computing the greatest unfounded set of a normal knowledge base is coNP-hard (\Proposition{nphard1}).
Because $\Atmost(T, F)$ can be computed in polynomial time, we conclude that the greatest unfounded set of a normal knowledge base $\K$ can be computed in polynomial time if $\Atmost(T, F) \union \pi(\O)$ is consistent and the greatest unfounded set of a disjunctive knowledge base $\K$ can be computed in polynomial time if $\Atmost(T, F) \union \pi(\O)$ is consistent and $\K$ is head-independent.
Observe that for the knowledge base constructed in our proof of \Proposition{nphard1}, $\Atmost(T, F) \union \pi(\O)$ is inconsistent.
We formally demonstrate the intractability of computing $\GUS(T, F)$ for a disjunctive knowledge base when $\Atmost(T, F) \union \pi(\O)$ is consistent but the head-independence condition is not met.
\begin{proposition}
Let $\K = (\O, \P)$ be a disjunctive hybrid \mknf knowledge base such that the entailment relation $\OBone{S} \models a$ can be checked in polynomial time for any set $S \subseteq \KA(\K)$ and for any \Katom{} $\bfK a \in \KA(\K)$.
Let $(T, F)$ be a dependable partition of $\KA(\K)$ such that $\Atmost(T, F) \union \pi(\O)$ is consistent.
Determining whether a \Katom{} $\bfK a \in \KA(\K)$ is an unfounded atom of $\K$ w.r.t. $(T, F)$ is coNP-hard.
\end{proposition}

\begin{proof}
Let $SAT$ be an instance of 3SAT in conjunctive normal form such that $CLAUSE = \{ c_1, c_2, ..., c_n \}$ is the set of clauses in $SAT$, and $VAR =  \{ v_1, v_2, ..., v_n \}$ is the set of variables in $SAT$.
We construct a disjunctive hybrid \mknf knowledge base $\K = (\O, \P)$ s.t.
\begin{multline}
\begin{aligned}
\O = & \{ (v^{f}_{i} \lor v^{t}_{i}) \lxor v^{u}_{i} \setcomp \textrm{ for each $v_i \in VAR$ where $\lxor$ is exclusive-or } \} \union \\
     & \{ (v_i^{f} \land v_i^{t}) \implies sat \setcomp \textrm{ for each $v_i \in VAR$}  \} \union \\
     & \{ \left( (\bigwedge\limits_{c_i \in CLAUSE} clause_i) \land (\bigwedge\limits_{v_i \in VAR} \neg v^{u}_i) \right) \implies sat \setcomp \textrm{ where $clause_i$ is a formula}\\
     & \textrm{ obtained by replacing all occurrences of $v_i$ and $\neg v_i$ in $c_i$ with $v^{t}_{i}$ and $v^{f}_{i}$ respectively} \}
\end{aligned}
\end{multline}
and
\begin{multline}
\begin{aligned}
\P =  \{ \bfK sat \leftarrow \bfK sat  \} \union \bigcup \{ \{ (\bfK v_i^{t}, \bfK v_i^{f} \leftarrow) \} \setcomp v_i \in VAR \}
\end{aligned}
\end{multline}

Let $(T, F) = (\emptyset, \emptyset)$ and observe that $\Atmost(T, F) \union \pi(\O)$ is consistent ($\Atmost(T, F) = \KA(\K) \setminus \{ \bfK sat \}$).
We show that (1) For any \Katom{} $\bfK a$ and set of \Katoms{} $S$, the entailment relation $\OBone{S} \models a$ is computable in polynomial time and (2) that $\bfK sat$ is an unfounded atom of $\K$ w.r.t. $(\emptyset, \emptyset)$ if and only if $SAT$ is unsatisfiable.

(1) Observe that $\KA(\K) \union \pi(\O)$ is consistent, therefore, $S \union \pi(\O)$ is consistent for any set of \Katoms{} $S \subseteq \KA(\K)$.
The entailment relation $\OBone{S} \models v_i^{t}$ (resp. $\OBone{S} \models v_i^{f}$) holds if and only if $v_i^{t} \in S$ (resp. $v_i^{f} \in S$).
What remains to show is that $\OBone{S} \models sat$ can be checked in polynomial time when $\bfK sat \not\in S$.
We call a set of \Katoms{} $S$ consistent if it does not contain both $\bfK v_i^{t}$ and $\bfK v_i^{f}$ for every variable $v_i \in VAR$.
If $S$ is not consistent, then we have $\OBone{S} \models sat$ due to the second set of formulas in $\O$.
We assume that $S$ is consistent. 
We call a set of \Katoms{} $S$ total if it contains either $\bfK v_i^{t}$ or $\bfK v_i^{f}$ for each variable $v_i \in VAR$.
We consider the cases where $S$ is total and where $S$ is not total.
If $S$ is not total, we can construct a consistent first-order interpretation of $S \union \pi(\O)$ such that $v_i^{u}$ is true for some $v_i \in VAR$, thus $\OBone{S} \not\models sat$ if $S$ is consistent and not total.
Now we assume that $S$ is total and it follows that $\bigwedge\limits_{v_i \in VAR} \neg v^{u}_i$ is satisfied in the third set of formulas in $\O$.
We refer to a model $M$ of $S \union \pi(\O)$ as a proper model if for every $v_i \in VAR$ we have $v_i^{f}$ (resp. $v_i^{t}$) to be false in $M$ if $v_i^{f} \not \in S$ (resp. $v_i^{f} \not \in S$).
Observe that for all models of $S \union \pi(\O)$ modulo proper models, $sat$ is true because of the second set of formulas in $\O$ (recall that $S$ is total and consistent).
Note that for each proper model $M$ we have $M \models v_i^{f} \lxor v_i^{t}$ (where $\lxor$ is exclusive-or) because $S$ is consistent.
The only case where $\OBone{S} \not\models sat$ is if 
We have $\OBone{S} \models sat$ if and only if $S$ satisfies every formula $clause_i$.
This can easily be checked in polynomial time.

(2) When determining whether the \Katom{} $\bfK sat$ is unfounded w.r.t. $(\emptyset, \emptyset)$, we must consider each way to select a head-cut $R$.
This part of the proof carries out almost identically to part 2 of our proof of \Proposition{nphard1}.
We only outline the key differences:
Rather than relying on $\head(R) \union \pi(\O)$ to be inconsistent if $R$ does not correspond to a satisfying assignment of $SAT$ like in our proof of \Proposition{nphard1}, we rely on there being a single model of $\head(R) \union \pi(\O)$ where $sat$ is false (See (1) for details on proper models).
This is enough to show that $\head(R) \union \OBone{\emptyset} \not\models sat$.
When only considering proper models of $\head(R) \union \pi(\O)$, we can ignore the second set of formulas in $\O$ because a set of rule atom pairs $R$ containing both $(r, v_i^{f})$ and $(r, v_i^{t})$ is not a valid head-cut.
In order to determine whether a \Katom{} $\bfK a$ is unfounded w.r.t. $(\emptyset, \emptyset)$, we must exhaustively check $\head(R) \union \pi(\O)$ for every head-cut $R$ and can conclude that $SAT$ is unsatisfiable.
If we know that $SAT$ is unsatisfiable, there cannot exist a head-cut $R$ that proves that $\bfK a$ is not an unfounded atom.
\end{proof}

Intuitively, head-independence means that using multiple atoms from the head of a rule in conjunction with $\O$ cannot derive atoms that cannot be derived using only a single atom from the head of each rule.
The head-independence property is violated by the knowledge base in \Example{gus_operator_eg_1} and it ensures that $\Atmost(T, F) \not= \GUS(T, F)$.
If we were to alter the knowledge base such that the rule $\bfK a, \bfK b \leftarrow$ were changed to the pair of rules $\bfK a \leftarrow \Not b$ and $\bfK b \leftarrow \Not a$ then $\K$ has head-independence.
We show formally that for a head-independent knowledge base, the $\Z$ operator computes the greatest unfounded set w.r.t. $(T, F)$.

\section{A DPLL-Based Solver}
\label{dpll}
In this section we formulate a DPLL-based solver. 
First, 
we construct a well-founded operator $W_{\K}^{(T, F)}$ using the greatest unfounded set approximator from the previous section:
\begin{align*}
    T_{\K}^{(T, F)}(X, Y) &= \{ \bfK a \setcomp \textrm{ where } \OBone{T \union X} \models a \textrm{ for some } \bfK a \in \KA(\K) \} \union \\
                             \{ \bfK a \setcomp \textrm{ where } & \head(r) \setminus F = \{ \bfK a \} \textrm{ and } \body(r) \sqsubseteq (T \union X, F \union Y) \textrm{ for some } r \in \P \}) \\
    W_{\K}^{(T, F)}(X, Y) &= (T_{\K}^{(T, F)}(X, Y) \union T, (\KA(\K) \setminus Z_{\K}^{(T, F)}(X, Y)) \union F)
\end{align*}

We show that this operator maintains the property shown in \Proposition{notunfounded_is_T}.
\begin{proposition}\label{operator_maintains_extensions}
    If a dependable partition $(T, F)$ can be extended to an \mknf model $M$, then the dependable partition $\textbf{lfp }W_{\K}^{(T, F)}(X, Y)$ can also be extended to $M$.
\end{proposition}
\begin{proof}
    It follows from Corollary \ref{unfounded_is_F} that if $(T, F)$ can be extended an \mknf model $M$, then $(T, F \union Z_{\K}^{(T, F)}(T, F))$ can be extended to $M$.
    What's left to show is that if $(T, F)$ can be extended to an \mknf model $M$, then $(T \union T_{\K}^{(T, F)}(T, F)), F)$ can be extended to $M$.
    Suppose that there is some \Katom{} $\bfK a$ in $T \intersect T_{\K}^{(T, F)}(T, F))$ such that $M \not\mknfmodels \bfK a$.
    Then we either have that $\OBone{T} \models a$, and thus $M \not\mknfmodels \pi(\O)$ or that $M \not\mknfmodels \bfK h$ for each $\bfK h \in \head(r)$ and thus $M \not\mknfmodels \pi(\P)$. Either case contradicts the assumption that $M$ is an \mknf model of $\K$.
\end{proof}

Following Ji et al. \cite{ji_well_founded_2017}, we construct an abstract solver in Algorithm \ref{algo1} that prunes the search space for solving by using the $W_{\K}^{(T, F)}$ operator.
The \textit{CHECK-MODEL} procedure checks whether the \mknf interpretation
$$\{ I \setcomp \textrm{where $I \models \pi(\O)$, $I \models t$ for each $\bfK t \in T$, and $I \models f$ for each $\bfK f \in T$} \}$$
is an \mknf model of $\K$ whenever the solver reaches a total dependable partition.
This procedure is analogous to the NP-oracle required to check a model of a disjunctive logic program \cite{ben-eliyahu_propositional_1994}. Further developments are required for a more precise definition of this procedure. 
\medskip
\begin{algorithm}\label{algo1}
    \caption{$solver(\K, (T, F))$}
    $(T, F) \leftarrow W_{\K}(T, F) \sqcup (T, F)$\;
    \If{$T \intersect F \not= \emptyset$}{
        \Return false\;
    }
    \ElseIf{$T \union F = \KA(\K)$}{
        \If{CHECK-MODEL(T, F)}{
            \Return true\;
        }
        \Else{
            \Return false\;
        }
    }
    \Else{
        choose a \Katom{} $\bfK a$ from $\KA(\K) \setminus (T \union F)$\;
        \If{$solver(\K, (T \union \{ \bfK a \}, F))$} {
            \Return true\;
        }
        \Else{
            \Return $solver(\K, (T, F \union \{ \bfK a \}))$\;
        }
    }
\end{algorithm}

\begin{proposition}\label{solver-models}
    Given a partial partition $(T, F)$ of $\KA(\K)$, the invocation of Algorithm \ref{algo1} $solver(\K, (\emptyset, \emptyset))$ will return $true$ if $(T, F)$ can be extended to an \mknf model of $\K$.
\end{proposition}

\begin{proof}
It follows from \Proposition{operator_maintains_extensions} that the extension of $(T, F)$ on the first line of the algorithm, $(T, F) \leftarrow W_{\K}(T, F)$, does not miss any models.
No models exist that induce a partition $(T, F)$ s.t. $T \intersect F \not=\emptyset$.
Without the use of the $W_{\K}(T, F)$ operator, the solver algorithm will explore every partition $(T, F) \subseteq \KA(\K) \times \KA(\K)$ where $T \intersect F = \emptyset$.
Thus, the usage of the $W_{\K}(T, F)$ operator simply prunes the search space. 
\end{proof}

Given \Proposition{solver-models}, it is easy to modify Algorithm \ref{algo1} to report models instead of returning a boolean value.

We have identified some fundamental challenges in computing unfounded sets for hybrid \mknf knowledge bases that make the problem intractable. The operator constructed by Ji et al. \cite{ji_well_founded_2017} computes a subset of the greatest unfounded set and we build on this approximation with an extension for programs with rules with disjunctive heads. 

\section{Related Work}
\label{related}

Ji et al. establish a definition of unfounded sets for normal hybrid \mknf knowledge bases and construct well-founded operators that can be directly embedded in a solver \cite{ji_well_founded_2017}.
We extend their work by introducing a definition of unfounded sets that handles disjunctive rules, rules that have multiple \Katoms{} in their heads.
Our extension borrows from the unfounded-set techniques outlined by Leone et al. \cite{leone_disjunctive_1997} for disjunctive logic programs but with a few noteworthy differences. Namely, our definition cannot be used directly for model-checking.
If the ontology in $\K$ is empty, our definition is equivalent to Leon et al.'s for unfounded-free partitions. Similarly, if $\K$ is a normal knowledge base, our definition is equivalent to Ji et al.'s definition.

Both Ji et al. and Leone et al. outline abstract solvers for finding models of their respective languages.
These solvers follow the DPLL paradigm of exploring the search space for a model.
Both solvers substantially prune their search space using unfounded sets.
Because the complexity of model-checking a disjunctive hybrid MKNF knowledge bases is greater than that of normal hybrid MKNF knowledge bases \cite{motik_reconciling_2010}, our abstract solver in this work consults a model checker after a total interpretation has been guessed. This differs from the solver described by Ji et al. which does not rely on a model checker \cite{ji_well_founded_2017}.
Leone et al.'s solver does not deepen its search on partial interpretations that assign unfounded atoms as true (partitions that cannot be extended to models) \cite{leone_disjunctive_1997}.
This aggressive pruning strategy requires, at each step of the solver, an invocation of an algorithm with a complexity of $\Delta_{2}^{P}[O(log\textrm{ }n)]$ \cite{leone_disjunctive_1997}.
Industry-grade solvers, such as Clingo \cite{gebser_conflict-driven_2012} or HEX \cite{eiter_computational_1995}, recognize the impracticality of enumerating all unfounded sets many times during the solving process and these solvers introduce approximations techniques. As a caveat of using approximations of unfounded sets, a solver may deepen its search on partial interpretations that cannot be extended to models.
Because we rely on approximations of greatest unfounded sets, we think it is reasonable for our solver to employ similar strategies used by practical solvers and include some partitions that cannot be extended to models in its search.

Both Clingo and HEX have additional support for external atoms, atoms whose truth is dependant on external sources.
Clingo 5 defines $\mathbb{T}$-stable semantics \cite{gebser_theory_2016} to reason about external atoms via external theories.
HEX defines semantics for external atoms using boolean functions that take a total interpretation as input \cite{eiter_computational_1995}.
For any hybrid MKNF knowledge base, models of the accompanying ontology must be monotonic \cite{motik_reconciling_2010}.
While it may be possible to encode the semantics of hybrid MKNF knowledge bases using either the HEX or Clingo extensions,
neither solution exploits the monotonicity of external sources and both support nonmonotonic models of the external theories.

\section{Conclusion}
\label{conclusion}

We've provided a definition of unfounded sets for disjunctive hybrid \mknf knowledge bases, studied its properties, and formulated an operator to compute a subset of the greatest unfounded set of a knowledge base. This leads to a DPLL-based solver where after each decision constraint propagation is carried out by computing additional true and false atoms on top  of the current partial partition.
Our methods can be directly embedded into a solver for a drastic increase in efficiency when compared to a guess-and-verify solver, the current state of art for reasoning with disjunctive hybrid \mknf knowledge bases.
The addition of ontologies to answer set programs brings new challenges, namely, there is a complexity increase in computing unfounded sets even in the case of normal hybrid MKNF knowledge bases.
We leave computing unfounded sets in light of inconsistencies that arise because of $\O$ to future work.

\printbibliography

@article{SMT-JACM,
author = "R. Nieuwenhuis and A. Oliveras and C. Tinelli",
title = "Solving {SAT} and {SAT Modulo Theories}: From an Abstract Davis-Putnam-Logemann-Loveland Procedure to {DPLL(T)}",
journal = "Journal of the ACM",
volume =  "53",
number = "6",
year = "2006",
pages = "937-977"
}

@article{motik_reconciling_2010,
	title = {Reconciling description logics and rules},
	volume = {57},
	issn = {00045411},
	url = {http://portal.acm.org/citation.cfm?doid=1754399.1754403},
	doi = {10.1145/1754399.1754403},
	pages = {1--62},
	number = {5},
	journaltitle = {Journal of the {ACM}},
	author = {Motik, Boris and Rosati, Riccardo},
	urldate = {2020-03-26},
	date = {2010-06-01},
	langid = {english}
}

@article{van_gelder_well-founded_1991,
	title = {The well-founded semantics for general logic programs},
	volume = {38},
	issn = {0004-5411, 1557-735X},
	url = {http://dl.acm.org/doi/10.1145/116825.116838},
	doi = {10.1145/116825.116838},
	pages = {619--649},
	number = {3},
	journaltitle = {Journal of the {ACM}},
	author = {Van Gelder, Allen and Ross, Kenneth A. and Schlipf, John S.},
	urldate = {2020-03-20},
	date = {1991-07},
	langid = {english}
}

@inproceedings{lifschitz_nonmonotonic_nodate,
  author    = {Vladimir Lifschitz},
  editor    = {John Mylopoulos and
               Raymond Reiter},
  title     = {Nonmonotonic Databases and Epistemic Queries},
  booktitle = {Proceedings of the 12th International Joint Conference on Artificial
               Intelligence. Sydney, Australia, August 24-30, 1991},
  pages     = {381--386},
  publisher = {Morgan Kaufmann},
  year      = {1991},
  url       = {http://ijcai.org/Proceedings/91-1/Papers/059.pdf},
  bibsource = {dblp computer science bibliography, https://dblp.org}
}

@article{leone_disjunctive_1997,
	title = {Disjunctive Stable Models: Unfounded Sets, Fixpoint Semantics, and Computation},
	volume = {135},
	issn = {08905401},
	url = {https://linkinghub.elsevier.com/retrieve/pii/S0890540197926304},
	doi = {10.1006/inco.1997.2630},
	shorttitle = {Disjunctive Stable Models},
	pages = {69--112},
	number = {2},
	journaltitle = {Information and Computation},
	shortjournal = {Information and Computation},
	author = {Leone, Nicola and Rullo, Pasquale and Scarcello, Francesco},
	urldate = {2020-05-12},
	date = {1997-06},
	langid = {english}
}

@article{ben-eliyahu_propositional_1994,
	title = {Propositional semantics for disjunctive logic programs},
	volume = {12},
	issn = {1012-2443, 1573-7470},
	url = {http://link.springer.com/10.1007/BF01530761},
	doi = {10.1007/BF01530761},
	pages = {53--87},
	number = {1},
	journaltitle = {Annals of Mathematics and Artificial Intelligence},
	shortjournal = {Ann Math Artif Intell},
	author = {Ben-Eliyahu, Rachel and Dechter, Rina},
	urldate = {2020-07-20},
	date = {1994-03},
	langid = {english}
}

@article{eiter_computational_1995,
	title = {On the computational cost of disjunctive logic programming: Propositional case},
	volume = {15},
	issn = {1012-2443, 1573-7470},
	url = {http://link.springer.com/10.1007/BF01536399},
	doi = {10.1007/BF01536399},
	shorttitle = {On the computational cost of disjunctive logic programming},
	pages = {289--323},
	number = {3},
	journaltitle = {Annals of Mathematics and Artificial Intelligence},
	shortjournal = {Ann Math Artif Intell},
	author = {Eiter, Thomas and Gottlob, Georg},
	urldate = {2020-09-18},
	date = {1995-09},
	langid = {english}
}

@book{sipser_introduction_1996,
	edition = {First edition},
	title = {Introduction to the Theory of Computation},
	isbn = {0-534-94728-X},
	publisher = {International Thomson Publishing},
	author = {Sipser, Michael},
	date = {1996}
}

@article{gebser_conflict-driven_2012,
	title = {Conflict-driven answer set solving: From theory to practice},
	volume = {187-188},
	issn = {00043702},
	url = {https://linkinghub.elsevier.com/retrieve/pii/S0004370212000409},
	doi = {10.1016/j.artint.2012.04.001},
	shorttitle = {Conflict-driven answer set solving},
	pages = {52--89},
	journaltitle = {Artificial Intelligence},
	shortjournal = {Artificial Intelligence},
	author = {Gebser, Martin and Kaufmann, Benjamin and Schaub, Torsten},
	urldate = {2020-09-19},
	date = {2012-08},
	langid = {english}
}

@inproceedings{gebser_theory_2016,
  author    = {Martin Gebser and
               Roland Kaminski and
               Benjamin Kaufmann and
               Max Ostrowski and
               Torsten Schaub and
               Philipp Wanko},
  editor    = {Manuel Carro and
               Andy King and
               Neda Saeedloei and
               Marina De Vos},
  title     = {Theory Solving Made Easy with Clingo 5},
  booktitle = {Technical Communications of the 32nd International Conference on Logic
               Programming, {ICLP} 2016 TCs, October 16-21, 2016, New York City,
               {USA}},
  series    = {{OASICS}},
  volume    = {52},
  pages     = {2:1--2:15},
  publisher = {Schloss Dagstuhl - Leibniz-Zentrum f{\"{u}}r Informatik},
  year      = {2016},
  url       = {https://doi.org/10.4230/OASIcs.ICLP.2016.2},
  doi       = {10.4230/OASIcs.ICLP.2016.2},
  bibsource = {dblp computer science bibliography, https://dblp.org}
}

@article{ji_well_founded_2017,
  author    = {Jianmin Ji and
               Fangfang Liu and
               Jia-Huai You},
  title     = {Well-founded operators for normal hybrid {MKNF} knowledge bases},
  journal   = {Theory Pract. Log. Program.},
  volume    = {17},
  number    = {5-6},
  pages     = {889--905},
  year      = {2017},
  url       = {https://doi.org/10.1017/S1471068417000291},
  doi       = {10.1017/S1471068417000291},
  bibsource = {dblp computer science bibliography, https://dblp.org}
}

\end{document}